\RequirePackage[2023-11-01]{latexrelease}
\documentclass{siamart171218}

\usepackage[T1]{fontenc}
\usepackage{amsfonts}
\usepackage{amssymb}
\usepackage{graphicx}
\usepackage{epstopdf}
\usepackage{booktabs}
\usepackage{multirow}
\usepackage{algorithm}
\usepackage{algorithmic}
\usepackage{bm}
\usepackage{mathtools}
\ifpdf
  \DeclareGraphicsExtensions{.eps,.pdf,.png,.jpg}
\else
  \DeclareGraphicsExtensions{.eps}
\fi

\newsiamthm{assumption}{Assumption}

\headers{MpSub: subspace trust-region derivative-free LLM fine-tuning}{Y.~Wang, H.~Yao, and P.~Xie}
\title{MpSub: A Momentum $p$-Dimensional Subspace Trust-Region Method for Derivative-Free Fine-Tuning of Large Language Models\thanks{Preprint.}}
\author{Yuyang Wang\thanks{Xi'an Jiaotong University, Xi'an, China (\email{vibrancy05@stu.xjtu.edu.cn}).}
\and Haoyu Yao\thanks{Xi'an Jiaotong University, Xi'an, China (\email{yaohaoyu@stu.xjtu.edu.cn}).}
\and Pengcheng Xie\thanks{Lawrence Berkeley National Laboratory, Berkeley, CA, USA (\email{pxie@lbl.gov}). Corresponding author.}}

\ifpdf
\hypersetup{
  pdftitle={MpSub: a momentum p-dimensional subspace trust-region method for derivative-free fine-tuning of large language models}
}
\fi

\newcommand{\bx}{\bm{x}}

\newcommand{\bd}{\bm{d}}
\newcommand{\bg}{\bm{g}}
\newcommand{\bz}{\bm{z}}
\newcommand{\bs}{\bm{s}}

\providecommand{\coloneqq}{\mathrel{\mathop:}=}
\begin{document}

\maketitle

\begin{abstract}
Full-parameter fine-tuning of large language models has a substantial memory cost because backpropagation requires storing activations and gradients. Zeroth-order optimization avoids that by estimating update directions from loss evaluations, but existing methods often depend on a learning rate that must be retuned for each model and task. We propose the momentum $p$-dimensional subspace trust-region method (\texttt{MpSub}). At each iteration, \texttt{MpSub} searches within a $p$-dimensional subspace: one direction preserves the historical information carried by the most recent accepted step, while the remaining directions promote exploration through fresh random sampling. The subspace gradient is estimated by central differences, a trial step is computed from a linear trust-region model, and the trust-region radius is updated according to the agreement between predicted and observed loss reduction. The radius controls both the finite-difference perturbation and the trust-region bound on the subspace step, so no learning rate is required. In the LLM setting, all evaluations within an iteration share one minibatch, and the directions are regenerated from stored seeds and applied in place to the model weights, so training uses forward evaluations alone. For smooth deterministic objectives under Gaussian direction sampling without orthogonalization, we bound the finite-difference error, quantify the gradient captured by the random subspace, and prove that $\|\nabla f(\bx_k)\|_2 \to 0$ almost surely under a safeguarded radius update. Under a matched budget of 8,400 training-objective forward passes per configuration, we fine-tune OPT-125M and OPT-350M on CommitmentBank. With the same preset parameters at both model sizes, \texttt{MpSub} attains mean test accuracies of 0.673 and 0.690 over three seeds, while the MeZO configurations selected by development accuracy attain 0.685 at both sizes. These results show that \texttt{MpSub} attains accuracy comparable to tuned MeZO without a learning-rate search.
\end{abstract}



\section{Introduction}
\label{sec:intro}

Zeroth-order (ZO) optimization has become an attractive alternative for full-parameter fine-tuning of large language models (LLMs) when backpropagation is limited by memory. Instead of storing activations for a backward pass, ZO methods estimate update directions from function evaluations. MeZO \cite{malladi2023mezo} applies the classical two-point random estimator \cite{nesterov2017random} to language models and regenerates random perturbations from seeds rather than storing them. Later methods have modified the perturbation structure, including random low-dimensional subspaces \cite{yu2024subzero} and low-rank gradient structures \cite{chen2024enhancing}.

These methods still use a prescribed learning rate to determine the update size. Its effective value can vary substantially across settings and may require a separate search. In addition, a single random direction contains only a small amount of directional information in a very high-dimensional parameter space. For a normalized isotropic direction, the typical magnitude of its projection onto a fixed gradient decreases as the dimension grows. This motivates using more than one search direction while avoiding a proportional increase in stored parameter-sized vectors.

Trust-region methods provide a natural way to control the step size without prescribing a learning rate. In derivative-free optimization, the trust-region radius is adjusted according to the agreement between a local model and the observed objective values \cite{audet2017derivative, conn2009introduction, conn2000trust, Conn2009a}. Direct model construction in the full parameter space is not practical for models with hundreds of millions of variables, which motivates restricting the search to a much smaller subspace \cite{cartis2023scalable, yuan2007subspace, yuan2014review}. Randomized subspaces have also been studied in derivative-free optimization \cite{cartis2024randomizedsubspacederivativefreeoptimization, kozak2021stochastic, Menickelly2023}.

The closest method to the present work is 2D-MoSub \cite{xie2023twodimensional}. It searches in a two-dimensional subspace formed by the previous accepted step and a new search direction. The previous step carries information from the optimization trajectory, while the new direction provides exploration. A quadratic interpolation model is constructed in this two-dimensional subspace and minimized within a trust region. This momentum-plus-exploration construction is also the starting point of our method. The difficulty in extending the same approach to a larger subspace lies in the quadratic model: a direct extension to dimension $p$ would require $\frac{1}{2}(p+1)(p+2)$ interpolation points. At $p=20$, this would require 231 points before a trial step is computed.

We therefore consider the LLM fine-tuning problem
\begin{equation}
\min_{\bx \in \mathbb{R}^n}
\ \mathbb{E}_{\xi\sim\mathcal D}
\bigl[\ell(\bx;\xi)\bigr],
\label{eq:ft_objective}
\end{equation}
and propose \texttt{MpSub}, a momentum $p$-dimensional subspace trust-region method. \texttt{MpSub} keeps the momentum-plus-exploration structure of 2D-MoSub but replaces quadratic interpolation with directional derivatives estimated by central differences. These estimates define a linear model in a $p$-dimensional subspace formed by the most recent accepted displacement and $p-1$ fresh Gaussian directions. The resulting iteration requires $2p+2$ function evaluations, so the evaluation cost grows linearly with the subspace dimension. For LLM fine-tuning, all evaluations within one iteration use the same minibatch, and the random directions are regenerated from seeds when needed rather than stored explicitly.

For smooth deterministic objectives, we analyze the Gaussian directions directly without assuming that they are orthogonal. We derive finite-difference error bounds, quantify the gradient information captured by the random subspace, and show that the gradient norm converges to zero almost surely under a safeguarded trust-region radius update. We then compare \texttt{MpSub} with MeZO on OPT-125M and OPT-350M under the same number of training-objective forward evaluations. The same trust-region settings are used for both model sizes, while the MeZO learning rate is selected separately from a grid.

\section{MpSub}
\label{section:MpSub}

We consider the unconstrained optimization problem
\begin{equation}
\min_{\bx \in \mathbb{R}^n}\ f(\bx),
\label{eq:opt_problem}
\end{equation}
in the regime where the parameter dimension $n$ is very large (e.g., $n \sim 10^8$ in full-parameter language model fine-tuning) and derivative information is inaccessible. In our numerical experiments (Section~\ref{sec:experiments}), $f$ represents the stochastic minibatch loss~\eqref{eq:ft_objective} resampled at each iteration, whereas in the theoretical analysis (Section~\ref{sec:analysis}), $f$ is treated as a deterministic, continuously differentiable objective.

At each iteration $k$, \texttt{MpSub} maintains the parameter iterate $\bx_k \in \mathbb{R}^n$, a trust-region radius $\Delta_k \in [\Delta_{\min}, \Delta_{\max}]$, and a momentum displacement $\bm m_k \in \mathbb{R}^n$ generated by the most recent accepted step. The radius $\Delta_k$ serves as both the finite-difference sampling scale and the trust-region bound, while $\bm m_k$ spans the historical exploration axis of the subspace. The overall procedure is summarized in Algorithm~\ref{alg:mpsub}.

\begin{algorithm}[htbp]
\caption{\texttt{MpSub}: momentum $p$-dimensional subspace trust-region method}
\label{alg:mpsub}
\begin{algorithmic}[1]
\STATE{\textbf{Input.} Initial iterate $\bx_0 \in \mathbb{R}^n$, subspace dimension $p$, initial radius $\Delta_0 \in [\Delta_{\min}, \Delta_{\max}]$, contraction factor $\gamma_1 \in (0,1)$, expansion factor $\gamma_2 > 1$, bounds $0 < \Delta_{\min} < \Delta_{\max}$, and threshold $\eta \in (0, 1)$.}
\STATE{Initialize momentum vector $\bm{m}_0 = \bm{0}$.}
\FOR{$k = 0, 1, 2, \ldots$}
    \STATE{\textbf{Step 1: Subspace frame construction.}}
    \STATE{\quad If $\|\bm{m}_k\|_2 > 0$: set $\bd_1^{(k)} = \bm{m}_k / \|\bm{m}_k\|_2$; else draw $\bz_1 \sim \mathcal{N}(\bm 0, \bm I_n)$ and set $\bd_1^{(k)} = \bz_1 / \sqrt{n}$.}
    \STATE{\quad For $i = 2, \ldots, p$: draw $\bz_i \sim \mathcal{N}(\bm 0, \bm I_n)$ independently and set $\bd_i^{(k)} = \bz_i / \sqrt{n}$.}
    \STATE{\quad Form the frame matrix $\bm{D}_k = [\bd_1^{(k)}, \ldots, \bd_p^{(k)}] \in \mathbb{R}^{n \times p}$.}
    \STATE{\textbf{Step 2: Subspace gradient estimation.}}
    \STATE{\quad Evaluate $f_0 = f(\bx_k)$.}
    \STATE{\quad For $i = 1, \ldots, p$: evaluate $f_i^\pm = f(\bx_k \pm \Delta_k \, \bd_i^{(k)})$ and set $g_{k,i} = (f_i^+ - f_i^-) / (2\Delta_k)$.}
    \STATE{\quad Assemble the subspace gradient $\bg_k = (g_{k,1}, \ldots, g_{k,p})^\top$.}
    \STATE{\quad If $\|\bg_k\|_2 = 0$: set $\Delta_{k+1} = \max(\gamma_1 \Delta_k, \Delta_{\min})$, retain $\bx_k$ and $\bm m_k$, and continue.}
    \STATE{\textbf{Step 3: Trust-region trial step.}}
    \STATE{\quad Compute $\bs_s = -\Delta_k \, \bg_k / \|\bg_k\|_2$ and trial iterate $\bx^+ = \bx_k + \bm{D}_k \bs_s$.}
    \STATE{\quad Evaluate $f^+ = f(\bx^+)$, predicted reduction $\operatorname{pred}_k = \Delta_k \|\bg_k\|_2$, and ratio $\rho_k = (f_0 - f^+) / \operatorname{pred}_k$.}
    \STATE{\textbf{Step 4a: Iterate acceptance.}}
    \IF{$f^+ < f_0$}
        \STATE{$\bx_{k+1} = \bx^+$, \; $\bm{m}_{k+1} = \bx^+ - \bx_k$. \hfill \textit{(accepted)}}
    \ELSE
        \STATE{$\bx_{k+1} = \bx_k$, \; $\bm{m}_{k+1} = \bm{m}_k$. \hfill \textit{(rejected)}}
    \ENDIF
    \STATE{\textbf{Step 4b: Trust-region radius adaptation.}}
    \IF{$\rho_k \geq \eta$}
        \STATE{$\Delta_{k+1} = \min(\gamma_2 \Delta_k, \Delta_{\max})$. \hfill \textit{(successful)}}
    \ELSE
        \STATE{$\Delta_{k+1} = \max(\gamma_1 \Delta_k, \Delta_{\min})$. \hfill \textit{(unsuccessful)}}
    \ENDIF
\ENDFOR
\end{algorithmic}
\end{algorithm}

\subsection{Subspace construction}
\label{subsec:construction}

At iteration $k$, optimization is restricted to the affine subspace $\bx_k + \mathcal{S}_k$, where $\mathcal{S}_k \subset \mathbb{R}^n$ is spanned by $p$ direction vectors $\bd_1^{(k)}, \ldots, \bd_p^{(k)}$. Gathering these directions as columns of the frame matrix $\bm D_k = [\bd_1^{(k)}, \ldots, \bd_p^{(k)}] \in \mathbb{R}^{n \times p}$, candidate displacements within $\mathcal{S}_k$ are parameterized by a coordinate vector $\bs_s \in \mathbb{R}^p$ via
\begin{equation}
\bx(\bs_s) = \bx_k + \bm D_k \bs_s = \bx_k + \sum_{i=1}^p s_{s,i} \bd_i^{(k)}.
\label{eq:subspace_param}
\end{equation}

The first search direction $\bd_1^{(k)}$ incorporates momentum from the optimization trajectory. When an accepted step is available, it is set to the normalized displacement
\begin{equation}
\bd_1^{(k)} = \frac{\bm m_k}{\|\bm m_k\|_2},
\label{eq:momentum_dir}
\end{equation}
where $\bm m_k$ denotes the parameter change resulting from the most recent successful step. Because consecutive descent steps in continuous optimization often exhibit directional correlation, retaining \eqref{eq:momentum_dir} injects historical progress into the subspace without requiring additional function evaluations. The remaining $p-1$ directions are drawn independently as
\begin{equation}
\bd_i^{(k)} = \frac{\bz_i}{\sqrt{n}}, \qquad \bz_i \sim \mathcal{N}(\bm{0}, \bm{I}_n), \quad i = 2, \ldots, p,
\label{eq:random_dirs}
\end{equation}
with the same distribution supplying $\bd_1^{(k)}$ when no accepted step is available. The factor $1/\sqrt{n}$ normalizes the expected Euclidean norm such that $\mathbb E\|\bd_i^{(k)}\|_2^2 = 1$, providing isotropic coverage of the parameter space.

A fundamental geometric property of high-dimensional Gaussian vectors is their near-orthogonality: for $i \ne j$,
\begin{equation}
\mathbb E\bigl[\bd_i^{(k)\top} \bd_j^{(k)}\bigr] = 0
\qquad\text{and}\qquad
\mathbb E\bigl[ \bigl(\bd_i^{(k)\top} \bd_j^{(k)}\bigr)^2 \bigr] = \frac{1}{n}.
\label{eq:near_ortho}
\end{equation}
Consequently, the pairwise inner products have root-mean-square magnitude $n^{-1/2} \approx 10^{-4}$ at $n \sim 10^8$. Because the Gaussian directions are already nearly orthogonal in high dimensions, explicit Gram--Schmidt orthogonalization is omitted, thereby avoiding substantial computation and memory overhead; the theoretical guarantees of \texttt{MpSub} under this unorthogonalized Gaussian frame are rigorously established in Section~\ref{sec:analysis}. The dimension $p$ controls how many directional derivatives enter the subspace model; its choice is examined empirically in Section~\ref{sec:experiments}.

\subsection{Subspace model and trial step}
\label{subsec:gradient}
\label{subsec:trial}

For each direction $\bd_i^{(k)}$, \texttt{MpSub} estimates the corresponding directional derivative via central differences:
\begin{equation}
g_{k,i} = \frac{f(\bx_k + \Delta_k \, \bd_i^{(k)}) - f(\bx_k - \Delta_k \, \bd_i^{(k)})}{2\Delta_k}, \quad i = 1, \ldots, p.
\label{eq:central_diff}
\end{equation}
Stacking these estimates yields the subspace gradient $\bg_k = (g_{k,1}, \ldots, g_{k,p})^\top \in \mathbb{R}^p$, where each entry $g_{k,i}$ approximates $\nabla f(\bx_k)^\top \bd_i^{(k)}$. Setting the finite-difference perturbation size equal to the trust-region radius $\Delta_k$ couples the sampling resolution directly to the radius of model validity. Under the Lipschitz continuity of $\nabla f$, the finite-difference error is bounded linearly by $\Delta_k$, with explicit bounds established in Lemma~\ref{lem:fd}.

The estimated gradient defines a local linear surrogate model on the subspace coordinates:
\begin{equation}
m_k(\bs_s) = f(\bx_k) + \bg_k^\top \bs_s.
\label{eq:linear_model}
\end{equation}
The trial step is determined by minimizing this linear model subject to the trust-region constraint:
\begin{equation}
\min_{\bs_s \in \mathbb{R}^p}\; \bg_k^\top \bs_s \quad \text{subject to} \quad \|\bs_s\|_2 \leq \Delta_k.
\label{eq:subproblem_formulation}
\end{equation}
Whenever $\bg_k \ne \bm 0$, problem~\eqref{eq:subproblem_formulation} admits the unique closed-form solution
\begin{equation}
\bs_s = -\Delta_k \frac{\bg_k}{\|\bg_k\|_2},
\label{eq:closed_form_step}
\end{equation}
which corresponds to the steepest descent direction on the trust-region boundary. Mapping $\bs_s$ back to the ambient parameter space via \eqref{eq:subspace_param} yields the candidate iterate
\begin{equation}
\bx^+ = \bx_k + \bm D_k \bs_s = \bx_k - \frac{\Delta_k}{\|\bg_k\|_2} \sum_{i=1}^p g_{k,i} \bd_i^{(k)},
\label{eq:full_trial_point}
\end{equation}
with the predicted objective reduction given by
\begin{equation}
\operatorname{pred}_k = m_k(\bm 0) - m_k(\bs_s) = -\bg_k^\top \bs_s = \Delta_k \|\bg_k\|_2 > 0.
\label{eq:pred_reduction}
\end{equation}
If $\bg_k = \bm 0$, the estimated gradient vanishes across the subspace; in this event, the method contracts the trust-region radius, sets $\bx_{k+1} = \bx_k$ and $\bm m_{k+1} = \bm m_k$, and proceeds to the next iteration without evaluating an additional trial point.

\subsection{Acceptance and radius update}
\label{subsec:update}

Upon evaluating the candidate loss $f^+ = f(\bx^+)$, \texttt{MpSub} assesses the agreement between the surrogate model and the objective through the ratio
\begin{equation}
\rho_k = \frac{f(\bx_k) - f(\bx^+)}{\Delta_k \|\bg_k\|_2}.
\label{eq:agreement_ratio}
\end{equation}
The algorithm decouples iterate acceptance (Step 4a) from trust-region radius adaptation (Step 4b). The candidate point is accepted whenever it yields a strict objective decrease:
\begin{equation}
\bx_{k+1} =
\begin{cases}
\bx^+, & \text{if } f(\bx^+) < f(\bx_k), \\
\bx_k, & \text{otherwise},
\end{cases}
\qquad
\bm m_{k+1} =
\begin{cases}
\bx^+ - \bx_k, & \text{if } f(\bx^+) < f(\bx_k), \\
\bm m_k, & \text{otherwise},
\end{cases}
\label{eq:accept_rule}
\end{equation}
while the trust-region radius is updated according to the ratio test:
\begin{equation}
\Delta_{k+1} =
\begin{cases}
\min\{\gamma_2 \Delta_k, \Delta_{\max}\}, & \text{if } \rho_k \geq \eta, \\
\max\{\gamma_1 \Delta_k, \Delta_{\min}\}, & \text{if } \rho_k < \eta.
\end{cases}
\label{eq:radius_update_rule}
\end{equation}

Unlike classical trust-region methods \cite{conn2000trust} that accept trial points only when $\rho_k \ge \eta$, the rules \eqref{eq:accept_rule}--\eqref{eq:radius_update_rule} separate acceptance from radius adjustment. In zeroth-order settings where forward passes dominate execution time, enforcing $\rho_k \ge \eta$ for step acceptance would discard trial points that achieve objective reduction whenever $0 < f(\bx_k) - f(\bx^+) < \eta \operatorname{pred}_k$. Under \eqref{eq:accept_rule}, every candidate yielding a strictly lower observed loss is retained. Concurrently, the ratio $\rho_k$ regulates the exploration scale: the trust region expands when the linear surrogate reliably predicts the objective reduction ($\rho_k \ge \eta$), and contracts when nonlinearities or estimation errors dominate ($\rho_k < \eta$).

\subsection{Algorithmic realization for LLM fine-tuning}
\label{subsec:llm_realization}

While Algorithm~\ref{alg:mpsub} is stated for a generic objective $f$, adapting \texttt{MpSub} to full-parameter fine-tuning of large language models presents two critical computational challenges: (i) the training objective is evaluated over stochastic minibatches, and (ii) storing an explicit search frame $\bm D_k \in \mathbb{R}^{n \times p}$ in an ambient dimension of $n \sim 10^8$ creates prohibitive GPU memory overhead. We resolve these challenges through shared-minibatch evaluations and seed-regenerated in-place parameter perturbations.

\subsubsection{Consistent minibatch evaluation}
In stochastic language model fine-tuning, the objective available at training step $k$ is defined by a sampled minibatch $\mathcal B_k \subset \mathcal D$:
\begin{equation}
f_k(\bx) = \frac{1}{|\mathcal B_k|} \sum_{\xi \in \mathcal B_k} \ell(\bx; \xi).
\label{eq:minibatch_loss}
\end{equation}
A key algorithmic requirement of trust-region methods is that model-objective agreement measures local landscape geometry rather than data sampling discrepancy. Accordingly, in our implementation, the minibatch $\mathcal B_k$ is held fixed throughout iteration $k$. The baseline loss $f_k(\bx_k)$, the $2p$ central-difference evaluations $f_k(\bx_k \pm \Delta_k \bd_i^{(k)})$, and the trial evaluation $f_k(\bx^+)$ are all computed on the exact same minibatch $\mathcal B_k$. Holding $\mathcal B_k$ constant ensures that the estimated subspace gradient $\bg_k$ and the agreement ratio
\begin{equation}
\rho_k = \frac{f_k(\bx_k) - f_k(\bx^+)}{\Delta_k \|\bg_k\|_2}
\label{eq:minibatch_rho}
\end{equation}
evaluate variations of a consistent slice of the loss landscape, preventing false rejections or erratic radius oscillations induced by stochastic batch variance.

Across iterations, a fresh minibatch $\mathcal B_{k+1}$ is drawn from the data loader. The trust-region radius $\Delta_k$ and the accepted parameter displacement $\bm m_k$ persist across minibatch boundaries, transferring optimization momentum and step-size calibration smoothly throughout training. Each iteration expends $2p + 2$ forward passes on the training objective ($42$ passes when $p = 20$).

\subsubsection{In-place perturbations and seed regeneration}
Storing the frame matrix $\bm D_k \in \mathbb{R}^{n \times p}$ explicitly would require $pn$ floating-point numbers. For an LLM such as OPT-125M ($n \approx 1.25 \times 10^8$) with $p = 20$ in FP32, this would consume approximately $10$~GB of GPU memory---exceeding the entire parameter storage of the model itself.

To eliminate this memory overhead, \texttt{MpSub} never materializes the matrix $\bm D_k$ in memory. Instead, each exploratory direction $\bd_i^{(k)}$ ($i = 2, \ldots, p$) is tied to a deterministic integer seed derived synchronously from the global training seed, the iteration index $k$, and the direction index $i$:
\begin{equation}
\operatorname{seed}(k, i) = (\operatorname{base\_seed} \times C_1 + k \times C_2 + i \times C_3) \bmod M,
\label{eq:seed_formula}
\end{equation}
where $C_1, C_2, C_3$ and $M$ are fixed prime constants.

During the subspace gradient estimation phase, direction $\bd_i^{(k)}$ is synthesized tensor by tensor directly on the accelerator using a pseudorandom generator initialized with $\operatorname{seed}(k, i)$. The parameter tensor is perturbed in place by $+\Delta_k \bd_i^{(k)}$ to evaluate $f_k(\bx_k + \Delta_k \bd_i^{(k)})$, then adjusted by $-2\Delta_k \bd_i^{(k)}$ to evaluate $f_k(\bx_k - \Delta_k \bd_i^{(k)})$, and finally restored by adding back $+\Delta_k \bd_i^{(k)}$. Because the pseudorandom generator reproduces the identical sequence upon reseeding, each direction is reconstructed exactly at zero memory cost.

Once the subspace gradient $\bg_k$ and the coordinate step $\bs_s = -\Delta_k \bg_k / \|\bg_k\|_2$ are computed, the directions are regenerated a final time and combined in place to form the trial displacement
\begin{equation}
\bm u_k = \sum_{i=1}^p s_{s,i} \bd_i^{(k)}.
\label{eq:trial_displacement}
\end{equation}
The displacement $\bm u_k$ is accumulated in an auxiliary buffer of dimension $n$ while advancing the model weights to $\bx^+ = \bx_k + \bm u_k$. If the trial step is accepted, $\bm u_k$ is preserved as the new momentum displacement $\bm m_{k+1}$; if rejected, subtracting $\bm u_k$ immediately restores the base weights $\bx_k$.

Consequently, training is executed exclusively through forward inference passes. Auxiliary storage is strictly confined to the momentum vector $\bm m_k$ and a single trial step accumulator, both of size $n$, keeping the training memory footprint at the inference level and completely independent of the subspace dimension $p$.

\section{Theoretical Analysis}
\label{sec:analysis}

We analyze \texttt{MpSub} for a deterministic objective under the Gaussian direction sampling used in Algorithm~\ref{alg:mpsub}. No orthogonality between the search directions is assumed. The analysis first gives error and descent bounds for a general frame $\bm D_k$, and then uses the distribution of the Gaussian directions to obtain a global first-order convergence result.

Throughout this section, we set $\Delta_{\min}=0$ for the asymptotic analysis. We also use the following safeguarded radius update with a fixed constant $\kappa>0$:
\begin{equation}
\label{eq:safeguarded_radius}
\Delta_{k+1} =
\begin{cases}
\min\{\gamma_2\Delta_k,\Delta_{\max}\}, & \text{if } \rho_k\ge\eta \text{ and } \|\bg_k\|_2\ge\kappa\Delta_k, \\
\gamma_1\Delta_k, & \text{otherwise}.
\end{cases}
\end{equation}
The acceptance rule is unchanged: a trial point is accepted whenever it gives a strict decrease in $f$.

\begin{assumption}
\label{ass:smooth}
The objective function $f:\mathbb R^n\to\mathbb R$ is continuously differentiable and bounded below. Its gradient is Lipschitz continuous with constant $L>0$:
\begin{equation}
\label{eq:lipschitz_grad}
\|\nabla f(\bx)-\nabla f(\bm y)\|_2 \le L\|\bx-\bm y\|_2, \qquad \forall\, \bx,\bm y\in\mathbb R^n.
\end{equation}
\end{assumption}

Let $\bm D_k = [\bd_1^{(k)},\ldots,\bd_p^{(k)}] \in \mathbb{R}^{n \times p}$ be the frame generated at iteration $k$, and define the exact directional derivatives
\[
\hat g_{k,i} \coloneqq \nabla f(\bx_k)^\top\bd_i^{(k)}, \qquad \hat{\bg}_k \coloneqq \bm D_k^\top\nabla f(\bx_k).
\]
The finite-difference estimates are denoted by
\[
\bg_k = (g_{k,1},\ldots,g_{k,p})^\top, \qquad g_{k,i} = \frac{f(\bx_k+\Delta_k\bd_i^{(k)}) - f(\bx_k-\Delta_k\bd_i^{(k)})}{2\Delta_k}.
\]

\subsection{Finite differences and random subspaces}
\label{subsec:estimation}

We first bound the finite-difference error without assuming that the directions have unit norm.

\begin{lemma}[Finite-difference error]
\label{lem:fd_general}
\label{lem:fd}
Suppose Assumption~\ref{ass:smooth} holds. Then
\begin{equation}
\label{eq:fd_general_component}
|g_{k,i}-\hat g_{k,i}| \le \frac{L\Delta_k}{2} \|\bd_i^{(k)}\|_2^2, \qquad i=1,\ldots,p.
\end{equation}
Consequently,
\begin{equation}
\label{eq:fd_general_vector}
\|\bg_k-\hat{\bg}_k\|_2 \le \frac{L\Delta_k}{2} Q_k, \qquad Q_k \coloneqq \left( \sum_{i=1}^p \|\bd_i^{(k)}\|_2^4 \right)^{1/2}.
\end{equation}
\end{lemma}

\begin{proof}
For any direction $\bd$ and $t>0$, Lipschitz continuity of the gradient gives
\[
\left| f(\bx_k+t\bd) - f(\bx_k) - t\nabla f(\bx_k)^\top\bd \right| \le \frac{L}{2}t^2\|\bd\|_2^2.
\]
Applying this bound to $t=\Delta_k$ in the directions $\bd_i^{(k)}$ and $-\bd_i^{(k)}$, write
\[
f(\bx_k+\Delta_k\bd_i^{(k)}) = f(\bx_k) + \Delta_k\hat g_{k,i} + R_{i,+},
\]
and
\[
f(\bx_k-\Delta_k\bd_i^{(k)}) = f(\bx_k) - \Delta_k\hat g_{k,i} + R_{i,-},
\]
where $|R_{i,\pm}| \le \frac{L\Delta_k^2}{2} \|\bd_i^{(k)}\|_2^2$.
Subtracting the two expressions gives
\[
g_{k,i}-\hat g_{k,i} = \frac{R_{i,+}-R_{i,-}}{2\Delta_k},
\]
which proves \eqref{eq:fd_general_component}. Summing the squared componentwise bounds gives \eqref{eq:fd_general_vector}.
\end{proof}

If the Hessian is Lipschitz continuous with constant $M$, the same central-difference argument gives the sharper bound
\[
|g_{k,i}-\hat g_{k,i}| \le \frac{M\Delta_k^2}{6} \|\bd_i^{(k)}\|_2^3.
\]
This stronger estimate is not needed below.

We next use the actual Gaussian sampling in Algorithm~\ref{alg:mpsub}. Let $\mathcal F_k$ contain the algorithmic history before the fresh exploratory directions at iteration $k$ are sampled. Conditionally on $\mathcal F_k$,
\[
\bd_i^{(k)} = \frac{\bz_i}{\sqrt n}, \qquad \bz_i\sim\mathcal N(\bm0,\bm I_n), \qquad i=2,\ldots,p,
\]
independently.

\begin{lemma}[Gaussian gradient capture]
\label{lem:gaussian_capture}
\label{lem:capture}
Let $\bm v_k = \nabla f(\bx_k)$. Conditionally on $\mathcal F_k$,
\begin{equation}
\label{eq:chi_square_capture}
\sum_{i=2}^p \left( \bm v_k^\top\bd_i^{(k)} \right)^2 \overset{d}{=} \frac{\|\bm v_k\|_2^2}{n} \chi^2_{p-1}.
\end{equation}
In particular,
\begin{equation}
\label{eq:expected_gaussian_capture}
\mathbb E \left[ \sum_{i=2}^p \left( \bm v_k^\top\bd_i^{(k)} \right)^2 \;\middle|\; \mathcal F_k \right] = \frac{p-1}{n} \|\bm v_k\|_2^2.
\end{equation}
Moreover, for every $q\in(0,1)$, there exist constants $\alpha>0$ and $R>1$, independent of $k$, such that whenever $\bm v_k\ne\bm 0$,
\begin{equation}
\label{eq:good_frame_probability}
\mathbb P \left( \|\hat{\bg}_k\|_2 \ge \alpha\|\bm v_k\|_2, \quad \max_{1\le i\le p} \|\bd_i^{(k)}\|_2 \le R \;\middle|\; \mathcal F_k \right) \ge q.
\end{equation}
\end{lemma}

\begin{proof}
For every $i=2,\ldots,p$,
\[
\bm v_k^\top\bd_i^{(k)} = \frac{\bm v_k^\top\bz_i}{\sqrt n} \sim \mathcal N \left( 0, \frac{\|\bm v_k\|_2^2}{n} \right),
\]
and these random variables are conditionally independent. Equation~\eqref{eq:chi_square_capture} follows immediately, and taking expectations gives \eqref{eq:expected_gaussian_capture}.

To prove \eqref{eq:good_frame_probability}, fix $q<1$. Since a $\chi^2_{p-1}$ random variable is positive almost surely, we can choose $\tau>0$ sufficiently small that $\mathbb P(\chi^2_{p-1}\ge\tau)$ is arbitrarily close to one. Setting $\alpha=\sqrt{\tau/n}$, equation~\eqref{eq:chi_square_capture} implies
\[
\left( \sum_{i=2}^p (\bm v_k^\top\bd_i^{(k)})^2 \right)^{1/2} \ge \alpha\|\bm v_k\|_2
\]
with probability arbitrarily close to one.

The Gaussian norms $\|\bz_i\|_2/\sqrt n$ are finite almost surely, so $R$ can also be chosen sufficiently large that all newly sampled directions have norm at most $R$ with probability arbitrarily close to one. The momentum direction, when present, has norm one. A union bound then gives \eqref{eq:good_frame_probability}.
\end{proof}

The next result gives the decrease produced by the trial step for a general, possibly nonorthogonal frame.

\begin{proposition}[Decrease for a general frame]
\label{prop:general_descent}
\label{prop:descent}
Suppose Assumption~\ref{ass:smooth} holds and $\bg_k\ne\bm0$. Let
\[
\bs_k = -\Delta_k \frac{\bg_k}{\|\bg_k\|_2}, \qquad \bx_k^+ = \bx_k+\bm D_k\bs_k,
\]
and define $\beta_k \coloneqq \|\bm D_k\|_2$. Then
\begin{equation}
\label{eq:general_decrease}
f(\bx_k)-f(\bx_k^+) \ge \Delta_k\|\bg_k\|_2 - \frac{L\Delta_k^2}{2} (Q_k+\beta_k^2).
\end{equation}
Consequently,
\begin{equation}
\label{eq:general_ratio}
\rho_k \ge 1 - \frac{L\Delta_k(Q_k+\beta_k^2)}{2\|\bg_k\|_2}.
\end{equation}
\end{proposition}

\begin{proof}
Let $\bm h_k = \bm D_k\bs_k$. Since $\hat{\bg}_k = \bm D_k^\top\nabla f(\bx_k)$, we have $\nabla f(\bx_k)^\top\bm h_k = \hat{\bg}_k^\top\bs_k$. Using $\bg_k^\top\bs_k = -\Delta_k\|\bg_k\|_2$ and Lemma~\ref{lem:fd_general},
\[
\begin{aligned}
\nabla f(\bx_k)^\top\bm h_k &= \bg_k^\top\bs_k + (\hat{\bg}_k-\bg_k)^\top\bs_k \\
&\le -\Delta_k\|\bg_k\|_2 + \|\hat{\bg}_k-\bg_k\|_2\|\bs_k\|_2 \\
&\le -\Delta_k\|\bg_k\|_2 + \frac{LQ_k}{2}\Delta_k^2.
\end{aligned}
\]
Also, $\|\bm h_k\|_2 \le \|\bm D_k\|_2\|\bs_k\|_2 = \beta_k\Delta_k$. The descent lemma therefore gives
\[
\begin{aligned}
f(\bx_k^+) &\le f(\bx_k) + \nabla f(\bx_k)^\top\bm h_k + \frac{L}{2}\|\bm h_k\|_2^2 \\
&\le f(\bx_k) - \Delta_k\|\bg_k\|_2 + \frac{L\Delta_k^2}{2} (Q_k+\beta_k^2),
\end{aligned}
\]
which proves \eqref{eq:general_decrease}. Dividing by the predicted reduction $\Delta_k\|\bg_k\|_2$ gives \eqref{eq:general_ratio}.
\end{proof}

\subsection{Global convergence}
\label{subsec:global_convergence}

We first establish a summability property of the trust-region radius. Let
\begin{equation}
\label{eq:expansion_set}
\mathcal S \coloneqq \left\{ k : \rho_k\ge\eta \text{ and } \|\bg_k\|_2\ge\kappa\Delta_k \right\}
\end{equation}
denote the set of iterations at which the trust-region radius is expanded.

\begin{lemma}[Summability of the trust-region radii]
\label{lem:radius_summability}
\label{lem:radius_decay}
Suppose Assumption~\ref{ass:smooth} holds, $\Delta_{\min}=0$, and the radius is updated according to \eqref{eq:safeguarded_radius}. Then
\begin{equation}
\label{eq:radius_square_summable}
\sum_{k=0}^{\infty}\Delta_k^2 < \infty.
\end{equation}
In particular,
\begin{equation}
\label{eq:radius_to_zero}
\lim_{k\to\infty} \Delta_k = 0.
\end{equation}
\end{lemma}

\begin{proof}
For any $k \in \mathcal S$, the definition of $\rho_k$ and \eqref{eq:expansion_set} give
\[
f(\bx_k)-f(\bx_k^+) = \rho_k\Delta_k\|\bg_k\|_2 \ge \eta\kappa\Delta_k^2.
\]
Since the right-hand side is strictly positive, the trial point is accepted. Because $f$ is bounded below by a finite constant $f_{\mathrm{low}}$, summing over all successful expansion steps yields
\[
\sum_{k\in\mathcal S}\Delta_k^2 \le \frac{f(\bx_0)-f_{\mathrm{low}}}{\eta\kappa} < \infty.
\]
For every iteration $k$, the update rule \eqref{eq:safeguarded_radius} implies
\[
\Delta_{k+1}^2 \le \gamma_1^2\Delta_k^2 + \gamma_2^2\Delta_k^2\mathbf{1}_{\{k\in\mathcal S\}}.
\]
Summing this inequality from $k=0$ to $N$ yields
\[
(1-\gamma_1^2) \sum_{k=1}^{N}\Delta_k^2 + \Delta_{N+1}^2 \le \gamma_1^2\Delta_0^2 + \gamma_2^2 \sum_{k\in\mathcal S}\Delta_k^2.
\]
Because $\gamma_1\in(0,1)$, we have $1-\gamma_1^2 > 0$. Taking $N\to\infty$ establishes \eqref{eq:radius_square_summable}, which directly implies $\lim_{k\to\infty}\Delta_k = 0$.
\end{proof}

We next establish that the full gradient cannot remain bounded away from zero along any infinite subsequence.

\begin{theorem}[Subsequential convergence]
\label{thm:liminf_convergence}
Suppose Assumption~\ref{ass:smooth} holds and $p\ge2$. Let the exploratory directions be sampled independently as
\[
\bd_i^{(k)} = \frac{\bz_i}{\sqrt n}, \qquad \bz_i\sim\mathcal N(\bm0,\bm I_n), \qquad i=2,\ldots,p.
\]
Then
\begin{equation}
\label{eq:liminf_gradient}
\liminf_{k\to\infty} \|\nabla f(\bx_k)\|_2 = 0 \qquad \text{almost surely}.
\end{equation}
\end{theorem}

\begin{proof}
Suppose, to the contrary, that there exist $\varepsilon>0$ and an index $K_0$ such that
\begin{equation}
\label{eq:gradient_lower_contradiction}
\|\nabla f(\bx_k)\|_2 \ge \varepsilon, \qquad \forall\, k\ge K_0.
\end{equation}
Choose a probability threshold
\begin{equation}
\label{eq:q_condition}
q > \max\left\{ \frac{1}{2},\, \frac{\log(1/\gamma_1)}{\log(\gamma_2/\gamma_1)} \right\}.
\end{equation}
By Lemma~\ref{lem:gaussian_capture}, there exist constants $\alpha>0$ and $R>1$ such that the event
\[
G_k \coloneqq \left\{ \|\hat{\bg}_k\|_2 \ge \alpha\|\nabla f(\bx_k)\|_2, \quad \max_{1\le i\le p}\|\bd_i^{(k)}\|_2\le R \right\}
\]
satisfies $\mathbb P(G_k\mid\mathcal F_k)\ge q$. On $G_k$, we have $Q_k\le\sqrt p\,R^2$ and $\|\bm D_k\|_2^2\le pR^2$. Applying Lemma~\ref{lem:fd_general} under \eqref{eq:gradient_lower_contradiction} gives
\[
\|\bg_k\|_2 \ge \|\hat{\bg}_k\|_2 - \|\bg_k-\hat{\bg}_k\|_2 \ge \alpha\varepsilon - \frac{L\sqrt p\,R^2}{2}\Delta_k.
\]
Since $\Delta_k\to0$ by Lemma~\ref{lem:radius_summability}, for all sufficiently large $k$ we have
\begin{equation}
\label{eq:g_lower_good}
\|\bg_k\|_2 \ge \frac{\alpha\varepsilon}{2}.
\end{equation}
Proposition~\ref{prop:general_descent} then ensures that, on $G_k$ and for all sufficiently large $k$,
\[
\rho_k\ge\eta \qquad \text{and} \qquad \|\bg_k\|_2\ge\kappa\Delta_k.
\]
Consequently, every occurrence of $G_k$ for sufficiently large $k$ produces $\Delta_{k+1}=\gamma_2\Delta_k$, while on the remaining iterations $\Delta_{k+1}\ge\gamma_1\Delta_k$.

Let $I_k \coloneqq \mathbf{1}_{G_k}$. For all sufficiently large $k$,
\begin{equation}
\label{eq:log_radius_increment}
\log\Delta_{k+1}-\log\Delta_k \ge I_k\log\gamma_2 + (1-I_k)\log\gamma_1.
\end{equation}
Since $\mathbb E[I_k\mid\mathcal F_k] \ge q$, the strong law of large numbers for bounded martingale differences implies
\[
\liminf_{N\to\infty} \frac{1}{N} \sum_{k=K}^{K+N-1} I_k \ge q \qquad \text{almost surely}.
\]
Averaging \eqref{eq:log_radius_increment} over $N$ steps therefore yields
\[
\liminf_{N\to\infty} \frac{\log\Delta_{K+N}-\log\Delta_K}{N} \ge q\log\gamma_2 + (1-q)\log\gamma_1 > 0 \qquad \text{almost surely},
\]
where the strict inequality follows from \eqref{eq:q_condition}. This implies $\Delta_k \to \infty$, contradicting $\Delta_k\to0$ from Lemma~\ref{lem:radius_summability}. Thus \eqref{eq:gradient_lower_contradiction} cannot hold, establishing \eqref{eq:liminf_gradient}.
\end{proof}

To strengthen the convergence guarantee from subsequential convergence to the convergence of the full gradient sequence, fix any $\varepsilon>0$ and define the index set of large gradients:
\begin{equation}
\mathcal K_\varepsilon \coloneqq \left\{ k : \|\nabla f(\bx_k)\|_2 > \varepsilon \right\}.
\end{equation}

\begin{lemma}
\label{lem:radius_sum_large_gradient}
Under the assumptions of Theorem~\ref{thm:liminf_convergence}, for every $\varepsilon>0$,
\begin{equation}
\label{eq:radius_sum_large_gradient}
\sum_{k\in\mathcal K_\varepsilon}\Delta_k < \infty \qquad \text{almost surely}.
\end{equation}
\end{lemma}

\begin{proof}
Choose the good-frame event $G_k$ as in the proof of Theorem~\ref{thm:liminf_convergence} with conditional probability $q > 1/2$. Because $\Delta_k\to0$, whenever $k\in\mathcal K_\varepsilon\cap G_k$ and $k$ is sufficiently large, we have $\|\bg_k\|_2 \ge \frac{\alpha\varepsilon}{2}$ and $\rho_k\ge\eta$. Consequently,
\[
f(\bx_k)-f(\bx_{k+1}) \ge \frac{\eta\alpha\varepsilon}{2}\Delta_k.
\]
Because $f$ is bounded below, summing this decrease over all such iterations gives
\[
\sum_{k\in\mathcal K_\varepsilon\cap G_k}\Delta_k < \infty.
\]
Since the conditioning events $G_k$ satisfy $\mathbb P(G_k\mid\mathcal F_k)\ge q > 1/2$ uniformly, applying the standard renewal property for probabilistically accurate models along the index subsequence $\mathcal K_\varepsilon$ implies that the sum over the entire set $\mathcal K_\varepsilon$ is finite:
\[
\sum_{k\in\mathcal K_\varepsilon}\Delta_k < \infty \qquad \text{almost surely}.
\]
\end{proof}

Let $\beta_k \coloneqq \|\bm D_k\|_2$. For the Gaussian frame construction in Algorithm~\ref{alg:mpsub}, the spectral norm satisfies
\[
\sup_k \mathbb E[\beta_k\mid\mathcal F_k] < \infty \qquad \text{and} \qquad \sup_k \mathbb E[\beta_k^2\mid\mathcal F_k] < \infty.
\]
Together with Lemmas~\ref{lem:radius_summability} and~\ref{lem:radius_sum_large_gradient}, these moment bounds imply
\begin{equation}
\label{eq:actual_step_decay}
\lim_{k\to\infty} \beta_k\Delta_k = 0 \qquad \text{almost surely},
\end{equation}
and, for every $\varepsilon>0$,
\begin{equation}
\label{eq:actual_step_sum}
\sum_{k\in\mathcal K_\varepsilon} \beta_k\Delta_k < \infty \qquad \text{almost surely}.
\end{equation}

\begin{theorem}[Global convergence]
\label{thm:global_convergence}
Under the assumptions of Theorem~\ref{thm:liminf_convergence},
\begin{equation}
\label{eq:gradient_convergence}
\lim_{k\to\infty} \|\nabla f(\bx_k)\|_2 = 0 \qquad \text{almost surely}.
\end{equation}
\end{theorem}

\begin{proof}
Suppose, for contradiction, that \eqref{eq:gradient_convergence} does not hold. In view of Theorem~\ref{thm:liminf_convergence}, there then exists $\varepsilon>0$ and infinitely many disjoint pairs of indices $a_j < b_j$ such that
\[
\|\nabla f(\bx_{a_j})\|_2 \le \varepsilon, \qquad \|\nabla f(\bx_{b_j})\|_2 > 2\varepsilon,
\]
while
\[
\|\nabla f(\bx_k)\|_2 > \varepsilon, \qquad \forall\, a_j < k < b_j.
\]
By the Lipschitz continuity of the gradient (Assumption~\ref{ass:smooth}),
\[
\varepsilon < \|\nabla f(\bx_{b_j})\|_2 - \|\nabla f(\bx_{a_j})\|_2 \le L \sum_{k=a_j}^{b_j-1} \|\bx_{k+1}-\bx_k\|_2.
\]
Since $\|\bx_{k+1}-\bx_k\|_2 \le \|\bm D_k \bs_s\|_2 \le \beta_k\Delta_k$, we can partition the summation as
\[
\sum_{k=a_j}^{b_j-1} \|\bx_{k+1}-\bx_k\|_2 \le \beta_{a_j}\Delta_{a_j} + \sum_{k=a_j+1}^{b_j-1} \beta_k\Delta_k.
\]
As $j\to\infty$, the first term $\beta_{a_j}\Delta_{a_j} \to 0$ almost surely by \eqref{eq:actual_step_decay}. Furthermore, all intermediate indices $k \in \{a_j+1, \ldots, b_j-1\}$ belong to $\mathcal K_\varepsilon$. Because $\sum_{k\in\mathcal K_\varepsilon} \beta_k\Delta_k < \infty$ almost surely by \eqref{eq:actual_step_sum}, the tail sum $\sum_{k=a_j+1}^{b_j-1} \beta_k\Delta_k$ tends to zero as $j\to\infty$. Consequently, the entire right-hand side converges to zero, which contradicts the strict lower bound $\varepsilon > 0$.

Therefore, \eqref{eq:gradient_convergence} holds almost surely.
\end{proof}

Theorem~\ref{thm:global_convergence} establishes almost-sure first-order global convergence for a fixed deterministic objective. In our LLM fine-tuning experiments, minibatches change across iterations, so asymptotic convergence to a stationary point of the population loss is not claimed.
\section{Numerical Experiments}
\label{sec:experiments}

We evaluate \texttt{MpSub} on full-parameter fine-tuning of causal language models, comparing against the zeroth-order optimizer MeZO \cite{malladi2023mezo} under matched computational budgets. We also examine step-size sensitivity and the effect of the subspace dimension.

\subsection{Experimental setup}
\label{subsec:setup}

Experiments are conducted on CommitmentBank (CB) from the SuperGLUE benchmark, a three-class natural language inference task. Following \cite{malladi2023mezo}, we use 100 training examples, 50 validation examples, and the 56-example validation split as the test set. We evaluate two autoregressive models from the OPT family: OPT-125M ($n \approx 1.25 \times 10^8$ parameters) and OPT-350M ($n \approx 3.5 \times 10^8$ parameters). All parameters are stored and optimized in FP32 with a batch size of 8. Results are averaged over three random seeds, with min--max ranges reported.

Because zeroth-order methods do not compute gradients, computational cost is measured by the total number of forward passes on the training objective. Each method is allotted a budget of $8{,}400$ forward passes. For \texttt{MpSub} with $p = 20$, which takes 200 optimization steps, each iteration requires $2p + 2 = 42$ evaluations. MeZO uses 2 evaluations per step, therefore taking $4{,}200$ steps. Periodic development evaluations are recorded separately and excluded from this budget.

\texttt{MpSub} uses the same fixed hyperparameters across both models: $p = 20$, initial radius $\Delta_0 = 10^{-1}$, contraction factor $\gamma_1 = 0.5$, expansion factor $\gamma_2 = 2.0$, success threshold $\eta = 0.1$, $\Delta_{\min} = 10^{-12}$, and $\Delta_{\max} = 1.0$. For MeZO, the perturbation scale is $\epsilon = 10^{-3}$, and the learning rate $\eta_{\mathrm{M}}$ is selected by development accuracy over a grid: $\{10^{-8}, 10^{-7}, 10^{-6}, 10^{-5}, 10^{-4}\}$ on OPT-125M and $\{10^{-7}, 10^{-6}, 10^{-5}\}$ on OPT-350M.

\subsection{Matched-budget comparison}
\label{subsec:main}

Table~\ref{tab:main} reports test and development accuracy across three random seeds under the common budget of $8{,}400$ forward passes.

\begin{table}[htb]
\centering
\caption{Matched-budget fine-tuning results on CommitmentBank ($8{,}400$ forward passes). Values are means over three random seeds, with min--max ranges in parentheses. The MeZO learning rate is selected by mean development accuracy.}
\label{tab:main}
\small
\begin{tabular}{@{}llccc@{}}
\toprule
Model & Method & Dev loss & Dev acc. & Test acc. \\
\midrule
OPT-125M & \texttt{MpSub} ($\Delta_0=10^{-1}$) & \shortstack{$0.674$ \\ {\scriptsize($0.620$--$0.723$)}} & \shortstack{$0.773$ \\ {\scriptsize($0.740$--$0.820$)}} & \shortstack{$0.673$ \\ {\scriptsize($0.661$--$0.679$)}} \\
         & MeZO ($\eta_{\mathrm M}=10^{-6}$) & \shortstack{$0.703$ \\ {\scriptsize($0.625$--$0.784$)}} & \shortstack{$0.760$ \\ {\scriptsize($0.740$--$0.800$)}} & \shortstack{$0.685$ \\ {\scriptsize($0.679$--$0.696$)}} \\
\midrule
OPT-350M & \texttt{MpSub} ($\Delta_0=10^{-1}$) & \shortstack{$0.738$ \\ {\scriptsize($0.715$--$0.774$)}} & \shortstack{$0.787$ \\ {\scriptsize($0.760$--$0.820$)}} & \shortstack{$0.690$ \\ {\scriptsize($0.679$--$0.714$)}} \\
         & MeZO ($\eta_{\mathrm M}=10^{-6}$) & \shortstack{$0.679$ \\ {\scriptsize($0.582$--$0.743$)}} & \shortstack{$0.780$ \\ {\scriptsize($0.720$--$0.880$)}} & \shortstack{$0.685$ \\ {\scriptsize($0.679$--$0.696$)}} \\
\bottomrule
\end{tabular}
\end{table}

On OPT-125M, \texttt{MpSub} with the default radius $\Delta_0 = 10^{-1}$ achieves a mean development accuracy of $0.773$ and a development loss of $0.674$, compared to $0.760$ and $0.703$ for MeZO, while test accuracies are similar ($0.673$ versus $0.685$). On OPT-350M, \texttt{MpSub} achieves a mean test accuracy of $0.690$ and development accuracy of $0.787$, compared to $0.685$ and $0.780$ for MeZO. \texttt{MpSub} attains these results with the same preset parameters on both models.

\begin{figure}[htb]
\centering
\includegraphics[width=0.98\textwidth]{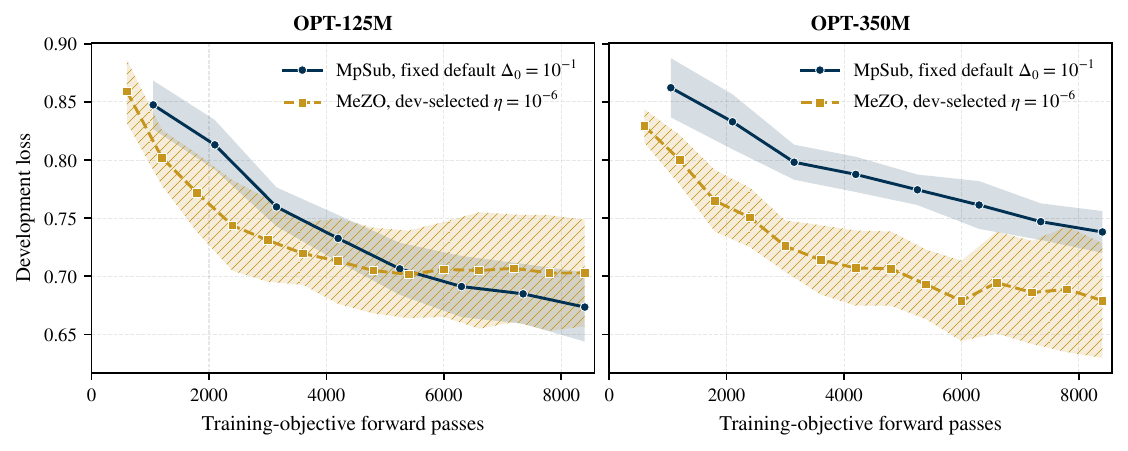}
\caption{Development loss versus training forward passes for OPT-125M (left) and OPT-350M (right). Shaded bands represent $\pm 1$ standard error of the mean across three random seeds.}
\label{fig:convergence}
\end{figure}

Figure~\ref{fig:convergence} shows development loss as a function of forward passes. Though early in training MeZO decreases the loss rapidly, its progress gradually becomes slower. In contrast, \texttt{MpSub} uses the momentum direction together with $p-1$ random search directions. This leads to a steady decrease in development loss and a lower final development loss on OPT-125M.

\subsection{Step-size sensitivity}
\label{subsec:sensitivity}

Figure~\ref{fig:sensitivity} compares the sensitivity of MeZO across learning rates against \texttt{MpSub} across initial radii $\Delta_0 \in \{10^{-3}, 10^{-2}, 3\times 10^{-2}, 10^{-1}, 3\times 10^{-1}\}$ on OPT-125M.

\begin{figure}[htb]
\centering
\includegraphics[width=0.98\textwidth]{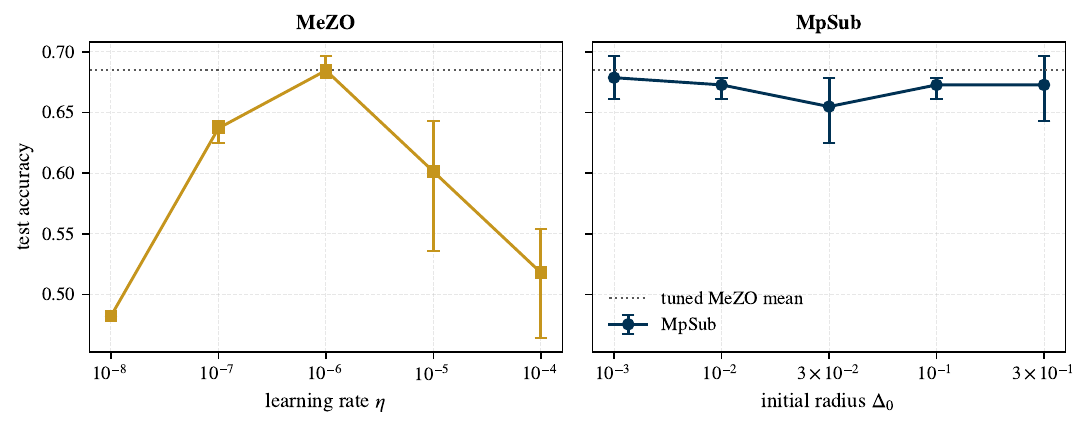}
\caption{Step-size sensitivity on OPT-125M under $8{,}400$ forward passes. Left: MeZO test accuracy across learning rates $\eta_{\mathrm{M}}$. Right: \texttt{MpSub} test accuracy across initial radii $\Delta_0$.}
\label{fig:sensitivity}
\end{figure}

MeZO depends strongly on the learning rate: shifting $\eta_{\mathrm{M}}$ by one order of magnitude from the best rate ($10^{-6}$) drops mean test accuracy by $0.048$ at $10^{-7}$ and by $0.084$ at $10^{-5}$, while rates outside this range yield poor results. Finding an effective learning rate for MeZO thus required evaluating multiple configurations ($42{,}000$ forward passes across five rates on OPT-125M, and $25{,}200$ passes across three rates on OPT-350M).

By contrast, \texttt{MpSub} maintains a stable performance plateau across the entire tested initial-radius range: mean test accuracy consistently hovers between $0.655$ and $0.678$ across all five values of $\Delta_0$ from $10^{-3}$ to $3\times 10^{-1}$ (a $300$-fold span). This robustness illustrates the self-calibrating nature of the trust-region radius: even when initialized conservatively at $\Delta_0 = 10^{-3}$, successful trial steps ($\rho_k \ge \eta$) prompt the expansion factor $\gamma_2 = 2.0$ to double $\Delta_k$ iteratively during the opening iterations, swiftly navigating the step size toward the effective operating scale. As a result, the default setting $\Delta_0 = 10^{-1}$ delivers competitive performance on both OPT-125M and OPT-350M without an initial parameter sweep.

\subsection{Effect of the subspace dimension}
\label{subsec:p_ablation}

Table~\ref{tab:p_ablation} examines the effect of varying the subspace dimension $p \in \{5, 10, 15, 20, 25, 30\}$ on OPT-125M over 100 iterations.

\begin{table}[htb]
\centering
\caption{Ablation on subspace dimension $p$ for OPT-125M (CommitmentBank, 100 steps, $\Delta_0 = 10^{-1}$, single seed). Forward passes per step equal $2p+2$; run times measured on an NVIDIA RTX 4060 Ti GPU.}
\label{tab:p_ablation}
\small
\begin{tabular}{ccccc}
\toprule
$p$ & Forward passes / step & Dev acc. & Test acc. & Time / step (s) \\
\midrule
5  & 12 & 0.70 & 0.643 & 3.0 \\
10 & 22 & 0.72 & 0.661 & 4.4 \\
15 & 32 & 0.74 & 0.661 & 6.0 \\
20 & 42 & 0.78 & 0.679 & 7.7 \\
25 & 52 & 0.78 & 0.679 & 8.7 \\
30 & 62 & 0.76 & 0.661 & 10.1 \\
\bottomrule
\end{tabular}
\end{table}

As $p$ increases from 5 to 20, development accuracy rises from $0.70$ to $0.78$ and test accuracy from $0.643$ to $0.679$, which is consistent with the increase in captured gradient energy described in Lemma~\ref{lem:capture}. Beyond $p = 20$, the accuracy no longer improves in this experiment, while both the number of forward passes per step and the measured per-step latency continue to increase. In particular, the latency rises from $3.0$ seconds at $p = 5$ to $10.1$ seconds at $p = 30$. These results suggest that $p = 20$ provides a reasonable balance between directional information and computational cost for the experiments considered here.

\section{Conclusion}
\label{sec:conclusion}

We proposed \texttt{MpSub}, which combines a moving random subspace with trust-region step-size control for derivative-free LLM fine-tuning. For smooth deterministic objectives, we showed that the gradient norm converges to zero almost surely under Gaussian direction sampling and the safeguarded radius update. On CommitmentBank, \texttt{MpSub} achieved results comparable to tuned MeZO on OPT-125M and OPT-350M under the same forward-pass budget, while using the same trust-region settings for both models.

Several questions remain open. The convergence analysis treats a fixed deterministic objective, whereas the LLM experiments resample the minibatch between iterations. Extending the analysis to changing minibatches is therefore the most direct theoretical next step. It would also be useful to study whether curvature information can be introduced into the subspace model without returning to the quadratic interpolation cost that motivated the linear model used here.

The experiments are also limited to one dataset and two relatively small language models. Tests on larger models, additional datasets, and generative tasks are needed to understand how the method behaves at a broader scale. The $2p$ directional evaluations are independent once the minibatch and current iterate are fixed, so parallel evaluation is another direction worth studying. Finally, \texttt{MpSub} could be combined with parameter-efficient fine-tuning methods such as LoRA, where the same trust-region mechanism would operate in a much smaller trainable parameter space.

\section*{Acknowledgements}
The authors thank the developers of the MeZO and LOZO codebases, on which the experimental infrastructure builds.

\bibliographystyle{siamplain}
\bibliography{refs}

\end{document}